%% file: MRmotion_v13_ogt.tex
\documentclass[journal]{IEEEtran}
\usepackage{cite}
\usepackage{subfig}
\usepackage{amsmath,amssymb,amsfonts,amsthm}
\usepackage{algorithmic}
\usepackage{graphicx}
\usepackage{hyperref}
\usepackage{multirow}
\usepackage{multicol}
\usepackage{booktabs}
\hypersetup{colorlinks=true,urlcolor=red}
\input{macros}

\newtheorem{theorem}{Theorem}
\newtheorem{proposition}[theorem]{Proposition}

\begin{document}

\title{Unsupervised MR Motion Artifact Deep Learning using Outlier-Rejecting Bootstrap Aggregation}
\date{\vspace{-4ex}}

\author{Gyutaek Oh, 
		Jeong Eun Lee, 
		and~Jong~Chul~Ye,~\IEEEmembership{Fellow,~IEEE}
\thanks{G. Oh, and J. C. Ye are with the Department of Bio and Brain Engineering, 
		Korea Advanced Institute of Science and Technology (KAIST), 
		Daejeon 34141, Republic of Korea (e-mail: \{okt0711, jong.ye\}@kaist.ac.kr). 
		J.E. Lee is with the Department of Radiology, Chungnam National University Hospital, Chungnam National University College of Medicine, 282 Munhwa-ro, Jung-gu, 
		Daejeon 35015, Republic of Korea (e-mail: leeje290@gmail.com).
		J.C. Ye is also with the Department of Mathematical Sciences, KAIST.}
\thanks{J.E. Lee and J.C. Ye are corresponding authors.} 
}

\maketitle

\begin{abstract}
Recently, deep learning approaches for MR motion artifact correction have been extensively studied.
Although these approaches have shown high performance and reduced computational complexity compared to classical methods, most of them require supervised training using paired artifact-free and artifact-corrupted images, which may prohibit its use in many important clinical applications.
For example, transient severe motion (TSM) due to acute transient dyspnea in Gd-EOB-DTPA-enhanced MR is difficult to control and model for paired data generation.
%
To address this issue, here we propose a novel unsupervised deep learning scheme through outlier-rejecting bootstrap subsampling and aggregation.
This is inspired by the observation that motions usually cause sparse $k$-space outliers in the phase encoding direction, so $k$-space subsampling along the phase encoding direction can remove some outliers and the aggregation step can further improve the results from the reconstruction network.
Our method does not require any paired data because the training step only requires artifact-free images.
Furthermore, to address the smoothing from potential bias to the artifact-free images, the network is trained in an unsupervised manner using
optimal transport driven cycleGAN.
We verify that our method can be applied for artifact correction from simulated motion as well as  real motion from TSM successfully, outperforming existing state-of-the-art deep learning methods.
\end{abstract}


\IEEEpeerreviewmaketitle

\section{Introduction}\label{sec:introduction}
\IEEEPARstart{M}{agnetic} resonance imaging (MRI) is a non-invasive imaging method which provides various contrast enhanced images without radiation exposure.
Although MRI has several advantages, the scan time of MRI is relatively long, 
so that motion artifacts from the patient's motion are often unavoidable.
In fact, motion artifact is considered as one of the main problems of MRI acquisition.

Many approaches have been investigated for MR motion artifact correction.
For example, additional devices are used to estimate motion, of which information is used for artifact correction \cite{ehman1984magnetic,todd2015prospective}.
Various sampling trajectories \cite{block2007undersampled,white2009motion,cruz2016accelerated,feng2016xd} and imaging sequences \cite{vasanawala2010navigated,chavhan2013abdominal,cruz2017highly} also have been studied to mitigate motion artifacts.
However, the requirement of additional hardware and scan time are the shortcomings of the above methods.
Moreover, many intrinsic motions, such as heartbeat and transient dyspnea, cannot be accurately measured with external devices.

One of the particular interests in this paper is the motion artifact from transient severe motion (TSM) in gadoxetic acid (Gd-EOB-DTPA)-enhanced MR.
More specifically, Gd-EOB-DTPA is a hepatocyte-specific MR contrast agent for the imaging diagnosis of hepatocellular carcinoma (HCC), liver metastases, and other diseases \cite{verloh2015liver}.
Its advantage lies in the offering hepatobiliary phase (HBP) imaging \cite{kubota2012correlation}.
Despite the advantages of HBP imaging in Gd-EOB-DTPA-enhanced MR, 
it has been reported that TSM due to acute transient dyspnea often occurred after the administration of Gd-EOB-DTPA \cite{davenport2013comparison}. 
Since TSM may have a serious effect on image quality during the arterial phase and affect the diagnostic accuracy of liver disease, a proper motion artifact correction algorithm is required for accurate detection and characterization of focal liver lesions \cite{ahn2010added}.  

Compressed sensing (CS) \cite{donoho2006compressed} based algorithms have been explored to solve the problem of MRI motion artifacts \cite{vasanawala2010improved,yang2013sparse,jin2017mri,kustner2017self}.
However, high computational complexity and difficulty in hyper-parameter tuning are limitations of these CS algorithms.
Furthermore, CS algorithms require raw $k$-space data, which is often difficult to acquire in practice.

Recently, deep learning approaches for MRI reconstruction have been extensively studied  \cite{hammernik2018learning,han2019k,wang2019accelerated,cha2020geometric}.
Deep learning methods for MRI motion artifact reduction \cite{duffy2018retrospective,pawar2018motion,zhang2019multi,liu2020motion,tamada2020motion} have also been proposed.
However, most of existing deep learning approaches for motion artifact correction are based on the simulated motion artifact data, which makes it 
difficult to apply them to the real MR situations.
Instead, other deep learning approaches exploited real motion artifact data obtained in controlled experiments \cite{armanious2019retrospective,kustner2019retrospective,armanious2020medgan}.
Although these approaches can reduce MR motion artifacts from similar real motions, it is difficult to obtain matched clean and artifact images in many important real-world applications: for example, in the arterial image degradation due to TSM in Gd-EOB-DTPA-enhanced MR, the motion-free paired images cannot be obtained.

To overcome the lack of paired data, Armanious et al proposed Cycle-MedGAN  \cite{armanious2019unsupervised}
and  Cycle-MedGAN V2.0 \cite{armanious2020unsupervised}.
These algorithms are trained on unpaired data set using cycleGAN \cite{zhu2017unpaired}.
Since these algorithms interpret the motion artifact removal as a style transfer problem,  there exists no explicit
motion artifact rejection mechanism so that  its applications to the real data set exhibit limited performance as will be shown later.

To address these issues,  here we propose a novel deep learning method for MR motion artifact correction that does not require any matched motion and motion-free images, but still offers high quality motion artifact correction by employing explicit motion artifact rejection mechanism.
Our method is based on a key observation that motion artifacts usually result in sparse $k$-space outliers in the phase encoding direction \cite{jin2017mri}.
In fact, one of the main ideas of our prior motion artifact removal algorithm  \cite{jin2017mri} is to use the sparse and low-rank Hankel matrix decomposition of $k$-space data to identify the outliers as a sparse component and replace them using a low-rank Hankel matrix completion.
Unfortunately,  the computation time is quite high for practical uses, and the outlier detection is quite sensitive to the choice of hyper-parameters.
As such,  for the motion artifact removal from TSM, the algorithm in \cite{jin2017mri} never succeeded.

Rather than relying on the difficult sparse outliers identification as done in \cite{jin2017mri}, the main idea of this work comes from $k$-space random subsampling, which can eliminate sparse outliers in $k$-space from motions in a probabilistic sense.
In fact, this is ideally fit to the deep learning framework thanks to its close relationship with bootstrap aggregation \cite{breiman1996bagging}, which has been recently explored in deep learning MR reconstruction \cite{cha2020geometric}.
In particular, if a neural network has been trained using motion-free training data as a reconstruction network from undersampled $k$-space data, the bootstrap subsampling in the $k$-space domain can eliminate some of the sparse $k$-space outliers from the motion, so that corresponding neural network creates images with fewer motion artifacts.
Then, during the aggregation step, the loss of image quality from the subsample data can be restored.
To deal with the potential bias from the use of the clean data for training, we use the unpaired training strategy using cycleGAN
for the reconstruction of high quality images from undersampled $k$-space data.
Since the neural network is trained using only clean data without motion artifact simulation and acquisition, our method is so flexible that can be applied to correct various motion artifact from both of simulated data and real data.
In particular, we successfully demonstrate that the intricate motions artifacts from TSM in Gd-EOB-DTPA-enhanced MR can be effectively and robustly removed using the proposed method as will be shown in experiments.

The remaining parts of our paper are organized as follows.
Section \ref{sec:related works} reviews previous deep learning methods for MRI motion artifact correction.
The main theory for the proposed method is then described in Section \ref{sec:theory}.
Next, Section \ref{sec:method} explains the experimental data sets, network architecture, and training details.
In Section \ref{sec:result}, our experimental results are shown.
Section \ref{sec:discussion} and Section \ref{sec:conclusion} contain the discussion and conclusion about our results.

\section{Related Works}\label{sec:related works}
Early deep learning approaches for MRI motion artifact correction are based on supervised learning.
For example, Duffy et al \cite{duffy2018retrospective} proposed convolutional neural networks (CNNs) for brain MR motion artifact correction using the adversarial loss, where dilated convolutions and skip connections are employed for the generator.
Pawar et al \cite{pawar2018motion} converted motion artifact problems to pixel classification problems, where the last layer classifies the pixels as one of 256 classes.
Zhang et al \cite{zhang2019multi} proposed a multi-scale network with residual blocks, and Liu et al \cite{liu2020motion} proposed densely connected multi-resolution blocks for MRI motion artifact reduction.
Also, Tamada et al proposed motion artifact reduction method based on CNN (MARC) inspired by Gaussian denoising with CNNs \cite{zhang2017beyond}.

Although aforementioned works showed improved performance, most of them used numerical simulations to generate motion artifact data.
The usage of simulated motion artifact data for network training is a limitation because the real motion artifact can differ from the simulated ones.

Some deep learning algorithms used real motion artifact data.
Armanious et al \cite{armanious2020medgan} suggested MedGAN for medical image translation, and applied it to
MRI motion artifact correction \cite{armanious2019retrospective,armanious2020medgan}.
MedGAN is based on pix2pix \cite{isola2017image},
so 
that it requires paired clean and motion artifact images, which is hard to obtain in practice.
To overcome the lack of paired data, Armanious et al \cite{armanious2019unsupervised} proposed Cycle-MedGAN, which is based on cycleGAN \cite{zhu2017unpaired}.
Cycle-MedGAN is trained on unpaired data set, and utilizes style loss and perceptual loss as cycle consistency loss.
Also, in the most recent work, they proposed Cycle-MedGAN V2.0 \cite{armanious2020unsupervised} to improve the performance of Cycle-MedGAN.
Although data sets that are used in \cite{armanious2019unsupervised} and \cite{armanious2020unsupervised} are unpaired and real motion artifact data, these motion artifact data are not realistic because the data were acquired with controlled motions, which is not common in the real situation.


\section{Theory}\label{sec:theory}
In this section, we first explain that the motion artifacts are often realized as $k$-space outliers, and then propose a novel deep learning scheme using the bootstrap subsampling and aggregation scheme.
    
\subsection{Motion Artifact as Sparse $k$-space Outliers}
According to the existing works related to motion artifacts \cite{bernstein2004handbook,hedley1991motion}, when there exist subject's transient motions, they incur displacements along the phase encoding direction at a few specific time instances, which result in the following $k$-space data:
\begin{align}
\widehat \xb_e(k_x,k_y) = 
\begin{cases} \widehat \xb(k_x,k_y)e^{-j \Phi(k_y)} & k_y \in \Kd \\
\widehat \xb(k_x,k_y)  & \mbox{otherwise,} \end{cases}
\label{eq:motion}
\end{align}
where $j=\sqrt{-1}$ and $\widehat \xb_e(k_x,k_y)$ and $\widehat \xb(k_x,k_y)$ refer to the motion-corrupted and motion-free $k$-space data, respectively, with $k_x$ and $k_y$ being the indices along the read-out and phase encoding directions, respectively.
Furthermore, $\Phi(k_y)$ is the displacement (in radian) at the phase encoding index $k_y$, and $\Kd$ denotes the phase encoding indices where the displacements occur.
Eq. \eqref{eq:motion} implies that motion artifacts cause phase variations, which appear as $k$-space sparse outliers along phase encoding direction.

More specifically, when MR images are acquired by 2D imaging, sparse outliers appear along the phase encoding direction as shown in Fig. \ref{fig:sparse_outlier}(a).
On the other hand, in 3D imaging, two-dimensional phase encoding steps are required and sparse outliers appear in the 3D volume of the $k$-space, as shown in Fig. \ref{fig:sparse_outlier}(b).
To address these sparse outliers in 3D imaging, 2D images along the sagittal direction should be processed.
%
However, it is also difficult to train the network by using the sagittal cross section images because the size and content of images are inconsistent depending on patients.
Moreover, the dimension of the phase encoding 1 direction in Fig. \ref{fig:sparse_outlier}(b) is usually significantly smaller than the phase encoding direction 2.
On the other hand, all of axial cross section images have the same size and consistent content.
Also, if 1D Fourier transform is applied to the 3D $k$-space volume along phase encoding direction 1, the sparse outliers in 3D volume can be converted into accumulated sparse oultiers in the $k$-space of 2D axial cross section, as shown in Fig. \ref{fig:sparse_outlier}(b).
Therefore, the sparse outliers can be considered as 1D outliers in 2D $k$-space along the phase encoding direction.
For this reason, our method can be designed to remove phase encoding directional outliers in the 2D $k$-space plane for both 2D and 3D imaging.

\begin{figure}[!h]
	\centerline{\includegraphics[width=0.8\linewidth]{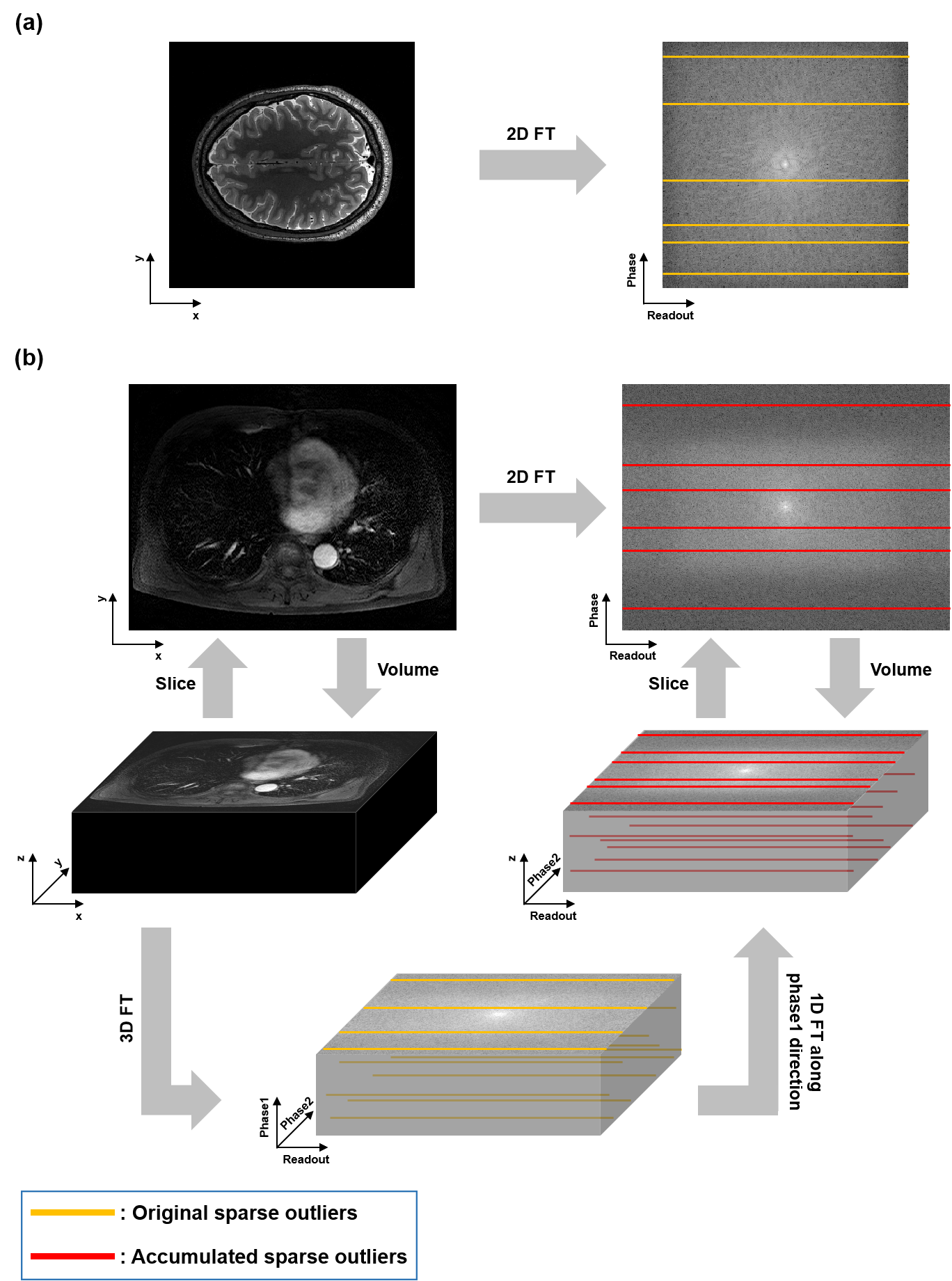}}
	\caption{Sparse outliers in $k$-space along the phase encoding direction:
	(a) 2D imaging, (b) 3D imaging}
	\vspace{-0.5cm}
	\label{fig:sparse_outlier}
\end{figure}

\subsection{Bootstrap Aggregation for Motion Artifact Correction}
Bootstrap aggregation is a classic machine learning technique which uses the bootstrap sampling and aggregation of the results to improve the accuracy of the base learner \cite{breiman1996bagging}.
The rationale for bootstrap aggregation is that it may be easier to train multiple simple weak learners and combine them into a more complex learner than to learn a single strong learner.
%
 
\begin{figure*}[!ht]
	\centerline{\includegraphics[width=0.7\linewidth]{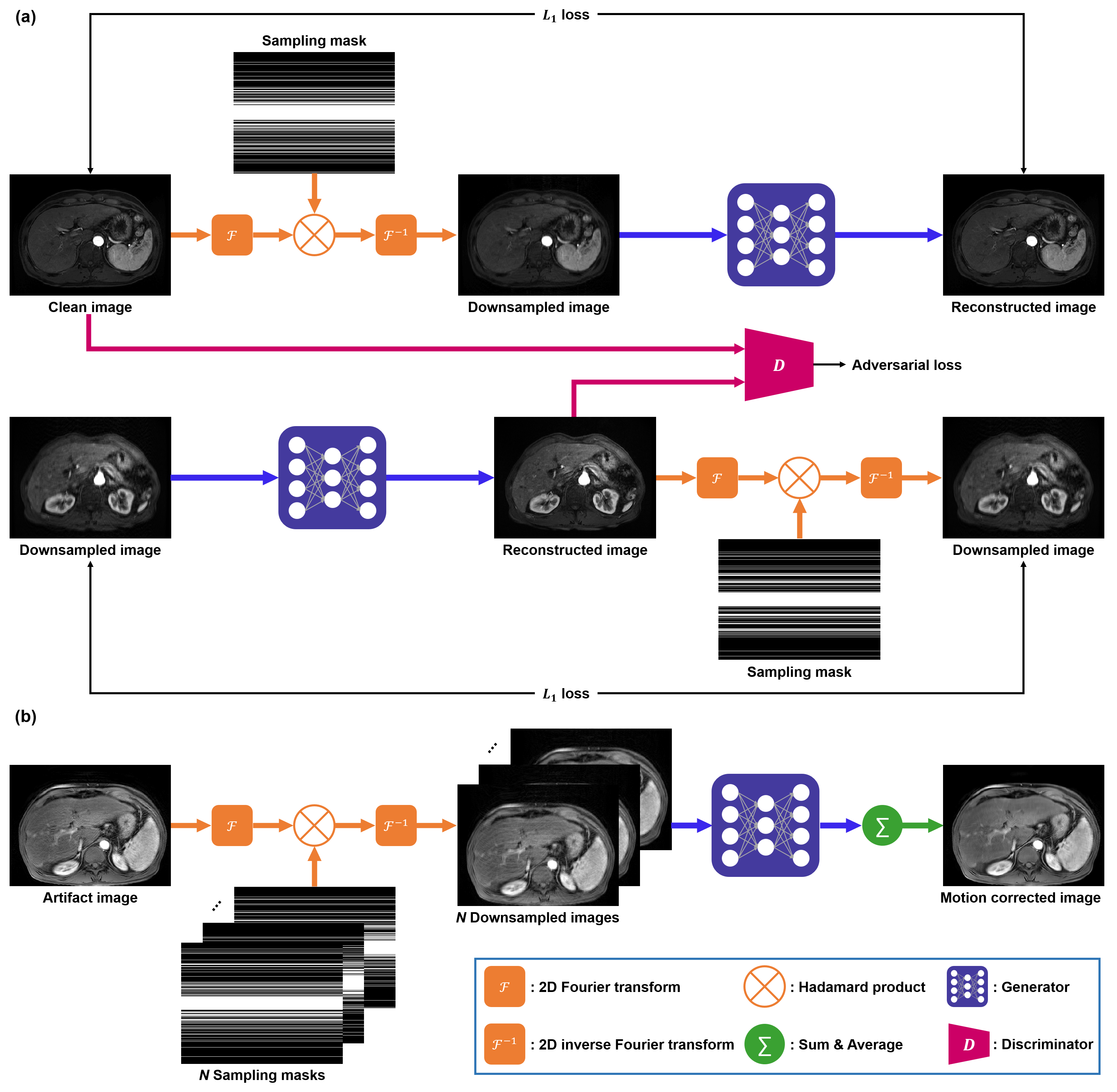}}
	\caption{Overall flow of the proposed method: 
		(a) In training phase, a clean image is downsampled in the $k$-space domain
		and the network is trained to convert the aliased image to the original fully sampled image.
		(b) In testing phase, $N$-aliased images are generated from one artifact image by random subsampling in the $k$-space domain.
		Then, each image is processed by the trained network.
		We then aggregate these reconstructed images so that the motion corrected image can be obtained.}
	\vspace{-0.5cm}
	\label{fig:training_testing}
\end{figure*}

In the context of MR deep learning, a bootstrap aggregation can be represented as follows \cite{cha2020geometric}: 
\begin{eqnarray}\label{eq:boost}
\widetilde\xb = \sum_{n=1}^N w_n G_\Theta(\Lc_n\widehat\xb),\quad 
\end{eqnarray}
where $\widehat\xb$ denotes the $k$-space data, $\widetilde \xb$ is the reconstructed images, $\Lc_n$ refers to the $n$-th $k$-space subsampling, $G_\Theta$ is a reconstruction network parameterized by $\Theta$ which reconstructs an image from a subsampled $k$-space data, and $w_n$ is the $n$-th weighting factor.
In other words, the final image is obtained as an aggregation of the individual reconstruction results from each $k$-space subsampling.
If the full $k$-space data is available, the benefit of bootstrap aggregation may seem unclear.
However, in our prior work \cite{cha2020geometric}, we demonstrated that instead of using just a single stronger deep learner, bootstrap aggregation can still provide high quality image reconstruction in the compressed sensing MRI.

One of the most important contributions of this work is to reveal that the benefit of bootstrap aggregation in \eqref{eq:boost} can be emphasized more in the context of motion artifact removal, since the sparse outlier model in \eqref{eq:motion} leads to the following key observation:
\begin{align}\label{eq:assume}
\Lc_n \widehat\xb \simeq \Lc_n \widehat\xb_e ,
\end{align}
for some sampling instance $\Lc_n$.
This implies that when using 1D subsampling along the phase encoding direction, the subsampling operation $\Lc_n$ can remove many sparse outliers of
$\widehat\xb_e$ in \eqref{eq:motion} so that it becomes similar to the subsampled $k$-space data from clean image.
Accordingly, it reduces the contribution of motion artifacts, so that the resulting bootstrap aggregation estimate in \eqref{eq:boost} may become much closer to the artifact-free image.
However, our assumption in \eqref{eq:assume} is only valid in probabilistic sense, since there may be the remaining artifact corrupted $k$-space
data after the subsampling. Therefore,  if our neural network is trained in a supervised manner using the clean data set,
there is a potential to introduce the bias.  In fact, as will be shown later in the experimental results, the supervised learning framework
introduces the blurring artifacts in the final reconstruction.
Therefore, we are interested in using an  unsupervised learning framework for the reconstruction from the subsampled $k$-space data,
and the optimal transport driven cycleGAN \cite{sim2019optimal,oh2020unpaired,cha2020unpaired} is nicely fit to this,  since it can exploit the deterministic subsampling
patterns using only a single generator.


The resulting motion artifact reduction method is illustrated in Fig. \ref{fig:training_testing}.
During the training phase in Fig. \ref{fig:training_testing}(a), the neural network is trained for accelerated MRI reconstruction using artifact-free data so that it can reconstruct a high quality image from any aliased image from randomly undersampled $k$-space data.
The comparison between the results of cycleGAN and supervised learning is discussed in Section \ref{sec:discussion}.

In particular, in the upper branch of cycleGAN, the clean images are first converted into the Fourier spectrum, which are then downsampled by $k$-space subsampling, so that downsampled images are generated by inverse Fourier transform.
These downsampled images become the input of the network.
In the lower branch, a downsampled image that is different from the upper branch image first pass through the network and then the output is downsampled in the Fourier domain.
Furthermore, the reconstructed clean image and the real clean image become inputs of the discriminator so that the discriminator distinguishes them as real or fake images.
Therefore, the network learns how to reconstruct downsampled images to realistic fully sampled images in an unsupervised manner.

On the other hand, images with motion artifacts are used at the test phase as illustrated in Fig. \ref{fig:training_testing}(b).
A motion artifact image is first converted to a Fourier domain data, and then several random subsampling are applied to obtain multiple aliased images.
Subsampling of $k$-space can delete some of $k$-space outliers with the phase error due to the motion.
Therefore, when the downsampled artifact images are reconstructed by the trained network, it is possible to obtain images with reduced motion artifact since the network was trained to reconstruct clean fully sampled images.
Then the neural network output of each aliased image is aggregated to obtain the final reconstruction.

Another important advantage of our method is that our method can also be used even when the $k$-space data is not available. Of course when the $k$-space data is available, the subsampling is applied directly to the real $k$-space data. However, in many practical applications,
the access to the $k$-space data is limited, so our image domain approach is very practical.


%


One may think that subsampling procedure introduces severe artifact images in the reconstruction so that the overall aggregation does not change the image quality.
However, 
the following proposition shows that although the estimation error for each individual subsampling pattern may be worse on average, the aggregated estimation always reduces the errors, which justifies the use of aggregation after the bootstrap subsampling.

\begin{proposition}\label{prp:PE_bagging}
Let $\sum_{n=1}^N w_n=1, w_n\geq 0$ and $\xb^*$ denotes the true image.
Then, we have
\begin{align*}
\sum_{n=1}^N w_n \|\xb^* - G_\Theta(\Lc_n\widehat \xb)\|^2 \geq  \|\xb^* -\sum_{n=1}^N w_n G_\Theta(\Lc_n\widehat\xb) \|^2 .
\end{align*}
\end{proposition}
\begin{proof}
Using $\sum_{n=1}^N w_n=1, w_n\geq 0$, we have
\begin{align*}
&\sum_{n=1}^N w_n \|\xb^* - G_\Theta(\Lc_n\widehat \xb)\|^2  \\
=& \|\xb^*\|^2 - 2\xb^{*\top}  \sum_n w_nG_\Theta(\Lc_n\widehat \xb) + \sum_n w_n \|G_\Theta(\Lc_n\widehat \xb) \|^2 \\
 \geq&   \|\xb^*\|^2 - 2\xb^{*\top}  \sum_k w_n G_\Theta(\Lc_n\widehat \xb) + \|\sum_n w_n G_\Theta(\Lc_n\widehat \xb)  \|^2\\
 =&\|\xb^* -\sum_n w_n G_\Theta(\Lc_n\widehat \xb) \|^2 , 
\end{align*}
where we use the Jensen's inequality for the inequality.
\end{proof}

\section{Method}\label{sec:method}
\subsection{Experimental Data Sets}
We use two data sets for our experiments.
The first data set is the human connectome project (HCP) data which contains MR images of the human brain.
The Siemens 3T system with 3D spin echo imaging was used to obtain the HCP data set.
The imaging parameters for the acquisition of the HCP data are as follows: echo train duration = 1105, TR = 3200 ms, TE = 565 ms, matrix size = $320\times320$, and voxel size = 0.7 mm $\times$ 0.7 mm $\times$ 0.7 mm.
The HCP data set is composed of only magnitude images.
This was used for quantitative verification of the performance of our method using simulated motions.
To acquire MR images with simulated motion artifact, $k$-space data is synthesized by taking the Fourier transform of the magnitude image.
We use 150 MR volumes for training and remaining 40 volumes for testing.
Each volume contains 20 slices of brain images, so training and test data contain 3000 and 800 MR slices, respectively.

The second data set is Gd-EOB-DTPA-enhanced magnetic resonance imaging of the liver collected from Chungnam National University Hospital.
Dynamic imaging, including the hepatic arterial phase, the portal phase, a 3-minutes transitional phase, and a 20-minutes hepatobiliary phase, was obtained using the fat-suppressed, breath-holding, T1-weighted 3D gradient recalled echo sequence using the following parameters: TR = 3.1 ms, TE = 1.5 ms, flip angle = 10$^\circ$, field of view = $256\times256$ mm$^2$, slice thickness/intersection gap = 2/0 mm, matrix =$320\times192$, number of excitation = 1, and acquisition time = 16.6 sec.
Since the data was provided in DICOM format, this data set also contains only magnitude images.
This liver MRI data set includes clean data from 23 patients, and artifact data from 20 patients.
Due to the administration of Gd-EOB-DTPA contrast agent, the motion artifact in liver MRI mainly occurs in the arterial phase.
Among 23 clean volumes, 18 volumes (3097 slices) were used to train our network, and the 20 patient artifact data (3412 slices) are used for testing to verify that our method can be applied to real motion artifact data.
In addition, we also generated synthetic motion data for quantitative evaluation using the other 5 clean volumes (888 slices) that were used as test data set for another simulation study.
Also, all images in this data set are cropped to the size of $384\times512$ to remove a margin of images.


\subsection{Network Architecture}
Fig. \ref{fig:training_testing} shows the schematic diagram of the proposed method, where we use the subsampling factor $N=15$ at the acceleration factor of $R=3$.
The network backbone architecture for our reconstruction network is depicted in Fig. \ref{fig:network}.
Our network architecture is based on U-Net \cite{ronneberger2015u} and it consists of convolution layers, instance normalization \cite{ulyanov2016instance}, activation layers and pooling layers.
Furthermore, we employ adaptive residual learning \cite{cha2020geometric} to improve reconstruction performance.
Specifically, the network's original output and the residual output are combined by the channel concatenation, from which the final output is generated using the last $1\times1$ convolution layer.
Also, we use the patchGAN discriminator \cite{zhu2017unpaired} similar to \cite{oh2020unpaired}.

\begin{figure}[!h]
	\centerline{\includegraphics[width=0.99\linewidth]{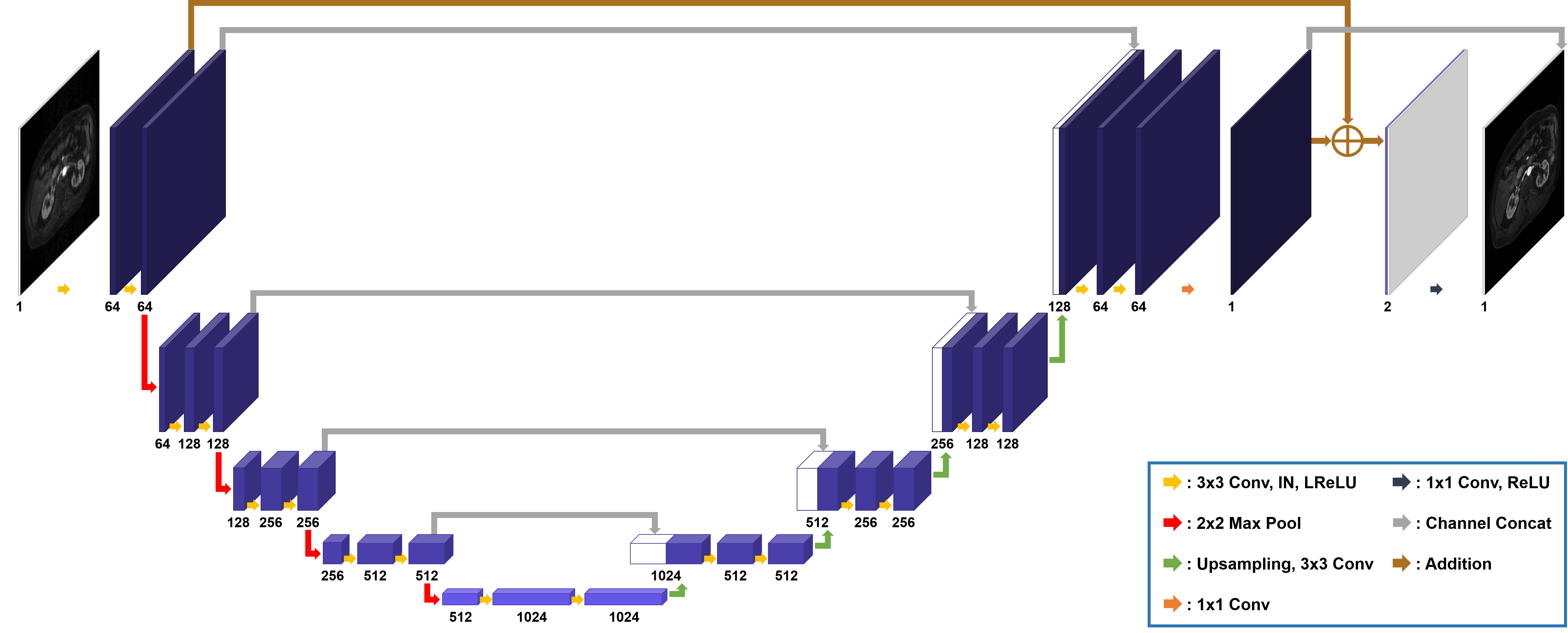}}
	\caption{Generator architecture.
		We use U-Net architecture and adaptive residual learning to reconstruct fully sampled images from downsampled images.
		Numbers below the blocks indicate the number of channels of each block.}
	\vspace{-0.5cm}
	\label{fig:network}
\end{figure}

\begin{figure*}[!t]
	\centerline{\includegraphics[width=0.7\linewidth]{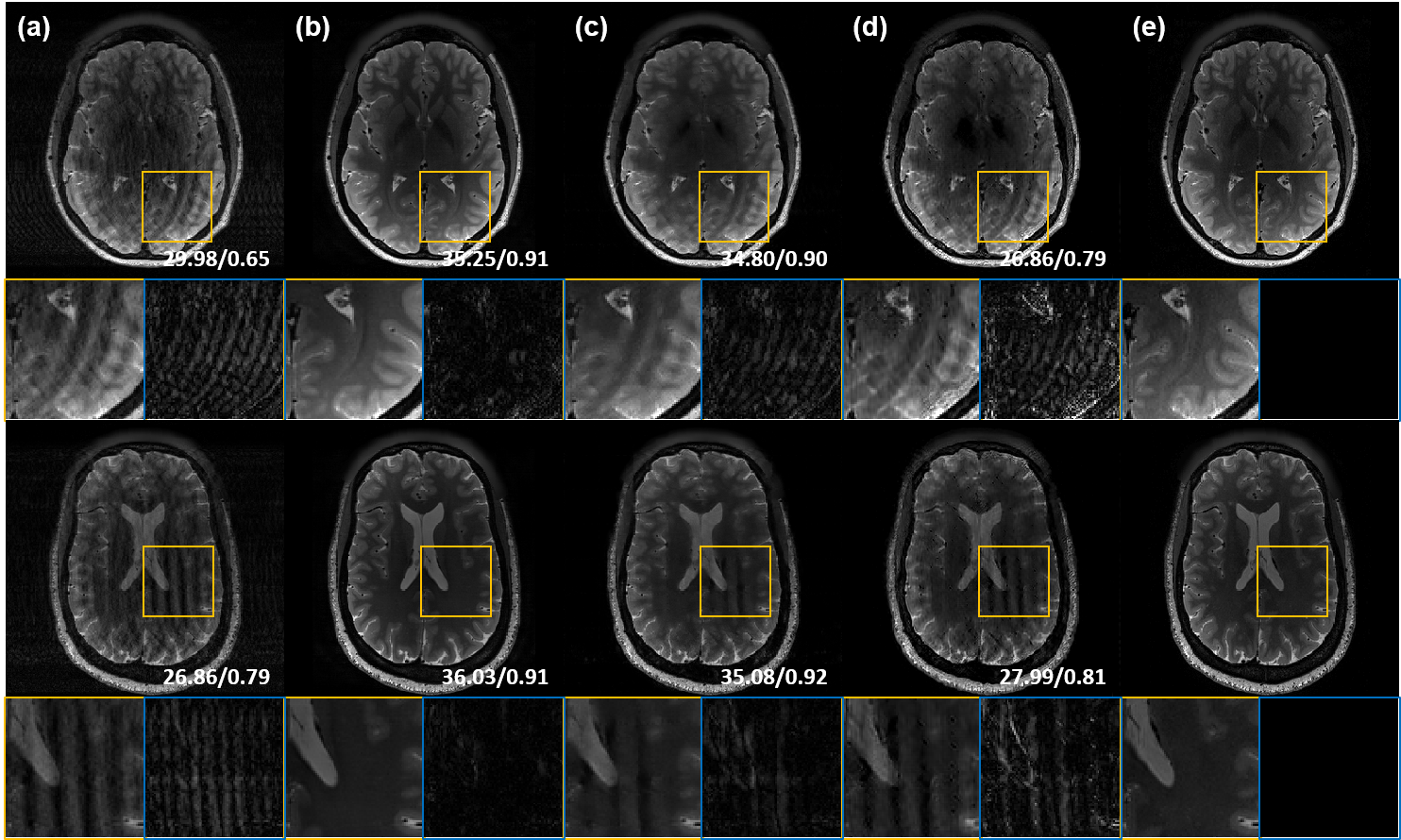}}
	\caption{Motion artifact correction results of brain data using various methods with simulated motion artifact:
		(a) Artifact images, (b) ours, (c) MARC, (d) Cycle-MedGAN V2.0, and (e) the ground-truth.
		The window level of images is adjusted for better visualization.
		PSNR and SSIM values of each image are shown in the corner of images.
		The difference images are amplified by a factor of three.}
	\vspace{-0.5cm}
	\label{fig:simulation_brain}
\end{figure*}

\subsection{Comparative Algorithms}
To verify the performance of our model, we use two state-of-the-art methods for MR motion artifact correction.
The first method is MARC \cite{tamada2020motion}, which is proposed to reduce the motion artifact for liver MRI using a convolutional neural network.
Because \cite{tamada2020motion} trained their model with simulated data, we also train MARC with simulated data, and test this model on both simulated and real motion artifact data.
For comparison, we only use arterial phase images to train MARC, since our method is also trained with arterial phase images.

Next, Cycle-MedGAN V2.0 \cite{armanious2020unsupervised} is employed as a comparison method.
Cycle-MedGAN V2.0 is an unpaired learning method which is based on cycleGAN \cite{zhu2017unpaired} for motion artifact correction and can therefore be trained with both real and simulated data.
Therefore we train and test Cycle-MedGAN V2.0 with both of real and simulated motion artifact data.
In \cite{armanious2020unsupervised}, they employed the pre-trained discriminator of a bidirectional GAN \cite{donahue2016adversarial} as a feature extractor.
However, we utilize VGG16 network \cite{simonyan2014very} pre-trained on the ImageNet data set because we cannot use a bidirectional GAN discriminator that is pre-trained on CT data set.
Since Cycle-MedGAN V2.0 requires real artifact data set, we collected additional 3096 real artifact images to train the cycleGAN network \cite{armanious2020unsupervised}.

To quantitatively assess the performance of various algorithms, we use the peak signal-to-noise ratio (PSNR) and the structural similarity index metric (SSIM) as quantitative metrics.
In addition, we conduct a clinical evaluation of the results of various methods using real motion artifact data.

\subsection{Artifact Simulation}
For quantitative evaluation, we use simulated motion artifact data.
For the generation of rigid motion artifacts, the phase error along the phase encoding direction can be formulated as \eqref{eq:motion} with
\begin{eqnarray}\label{eq:random simulated motion}
\Phi(k_y) = \begin{cases}
k_y\Delta_k, & |k_y|>k_0 \\
0, & \text{otherwise,}
\end{cases}
\end{eqnarray}
where $\Delta_k$ ($-37<\Delta_k<37$) is the degree of motion at $k$-space line $k$, and $k_0$ is delay time of the phase error due to the centric $k$-space filling \cite{tamada2020motion}.
Accordingly, to simulate random rigid motion for brain MRI data, we use \eqref{eq:random simulated motion} where $k_0$ is fixed to $\pi/10$, and $\Delta_k$ is randomly selected at each $k$-space line.

On the other hand, it is shown that the phase error caused by the breathing motion appears in $k$-space as a form of the sinusoidal function \cite{tamada2020motion}:
\begin{eqnarray}\label{eq:periodic simulated motion}
\Phi(k_y) = \begin{cases}
k_y\Delta \sin(\alpha k_y+\beta), & |k_y|>k_0 \\
0,& \text{otherwise,}
\end{cases}
\end{eqnarray}
where $\alpha$ ($0.1<\alpha<5$) and $\beta$ ($0<\beta<\pi/4$) are constants that determine the period and the phase shift, respectively, $\Delta$ ($0<\Delta<37$) is the number of pixels, which corresponds to 2.5$\sim$2.6 cm.
Hence, \eqref{eq:periodic simulated motion} is used for periodic motion artifact generation in liver MRI data.
We also set the value of $k_0$ as $\pi/10$ , and the values of other constants are chosen randomly in a certain range to simulate the realistic respiratory motion artifact.

\subsection{Training Details}
For bootstrap subsampling, we use a 1D Gaussian random sampling strategy along the phase encoding direction.
In the 1D Gaussian random sampling, lines near the center of $k$-space are sampled more than lines at the periphery.
We also set the acceleration factor $R$ to 3 or 4, and the autocalibration signal region contains 6 \% or 11 \% central $k$-space lines in the experiments using the brain and liver data set, respectively.

We divide each MR image by the standard deviation of its pixel values to normalize data.
To train our network, we use the Adam optimizer with momentum parameters $\beta_1=0.5$, $\beta_2=0.999$.
Also, we use $L_1$ loss with the batch size of 1.
The initial learning rate is set to $10^{-4}$, and reduced linearly to zero after 100 epochs.
Our model is trained for 200 epochs and implemented by TensorFlow.

\section{Experimental Results}\label{sec:result}
\begin{figure*}[!h]
	\centerline{\includegraphics[width=0.8\linewidth]{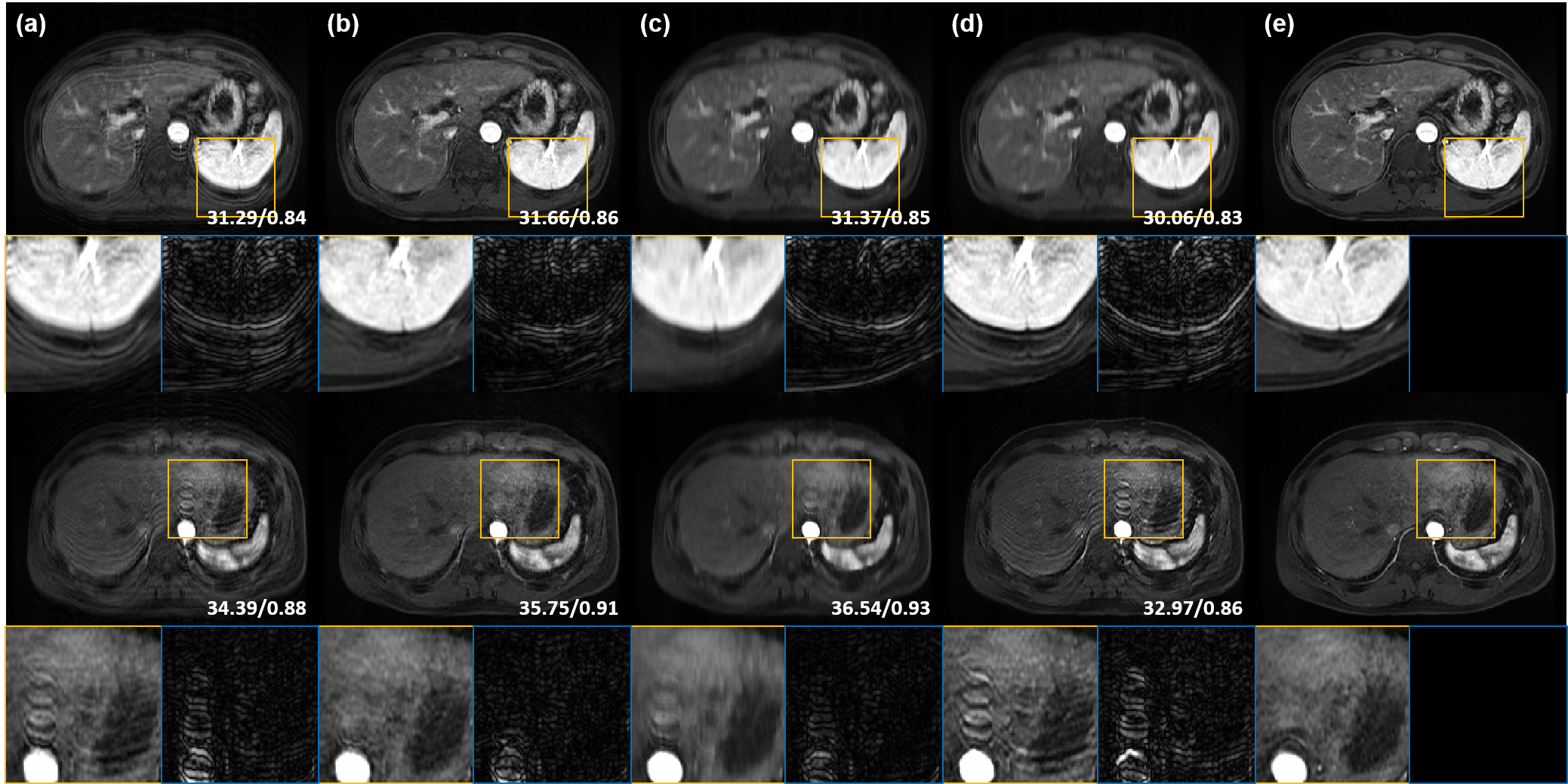}}
	\caption{Motion artifact correction results of liver data using various methods with simulated motion artifact: 
		(a) Artifact images, (b) ours, (c) MARC, (d) Cycle-MedGAN V2.0, and (e) the ground-truth.
		The window level of images is adjusted for better visualization.
		PSNR and SSIM values of each image are shown in the corner of images.
		The difference images are amplified by a factor of three.}
	\vspace{-0.5cm}
	\label{fig:simulation_liver}
\end{figure*}

\subsection{Experiments with Simulated Data}
Fig. \ref{fig:simulation_brain} shows the motion artifact correction results using comparison methods and our method.
As shown in Fig. \ref{fig:simulation_brain}(a), simulated random motion artifacts appear in the input images.
MARC \cite{tamada2020motion} reduces the motion artifacts in MR images and improves the quantitative metric values (Fig. \ref{fig:simulation_brain}(c)).
However, some severe artifacts still remain in the output images of MARC.
Next, the motion artifact correction fails when Cycle-MedGAN V2.0 \cite{armanious2020unsupervised} is applied, as shown in Fig. \ref{fig:simulation_brain}(d). 
Motion artifacts do not disappear and other artifacts such as black dots appear in the output of Cycle-MedGAN V2.0.
In addition, the PSNR values of the results using \cite{armanious2020unsupervised} are lower than PSNR values of the input images.
On the other hand, our method successfully corrects motion artifacts in MR images (Fig. \ref{fig:simulation_brain}(b)).
The motion artifacts that are not corrected by MARC are also reduced when using our method.
Furthermore, the proposed method shows similar quantitative metric values as MARC.

Next, we attempted to correct simulated periodic motion artifacts in liver MR images.
Fig. \ref{fig:simulation_liver} shows experimental results using various methods.
MARC reduces periodic motion artifacts and shows higher quantitative metric values than input images.
However, the output images from MARC are excessively blurred, making it difficult to recognize some anatomical structures or blood vessels as shown in Fig. \ref{fig:simulation_liver}(c).
Cycle-MedGAN V2.0 makes edges clear, but emphasizes motion artifacts together (Fig. \ref{fig:simulation_liver}(d)).
The emphasis of artifacts leads to a deterioration in the quantitative results.
Meanwhile, our method in Fig. \ref{fig:simulation_liver}(b) reduces motion artifacts in liver MR images and reconstructs anatomical structures and details.
Our method also shows comparable quantitative metric values with MARC which is based on supervised learning with paired data.

\begin{figure*}[!h]
	\centerline{\includegraphics[width=0.8\linewidth]{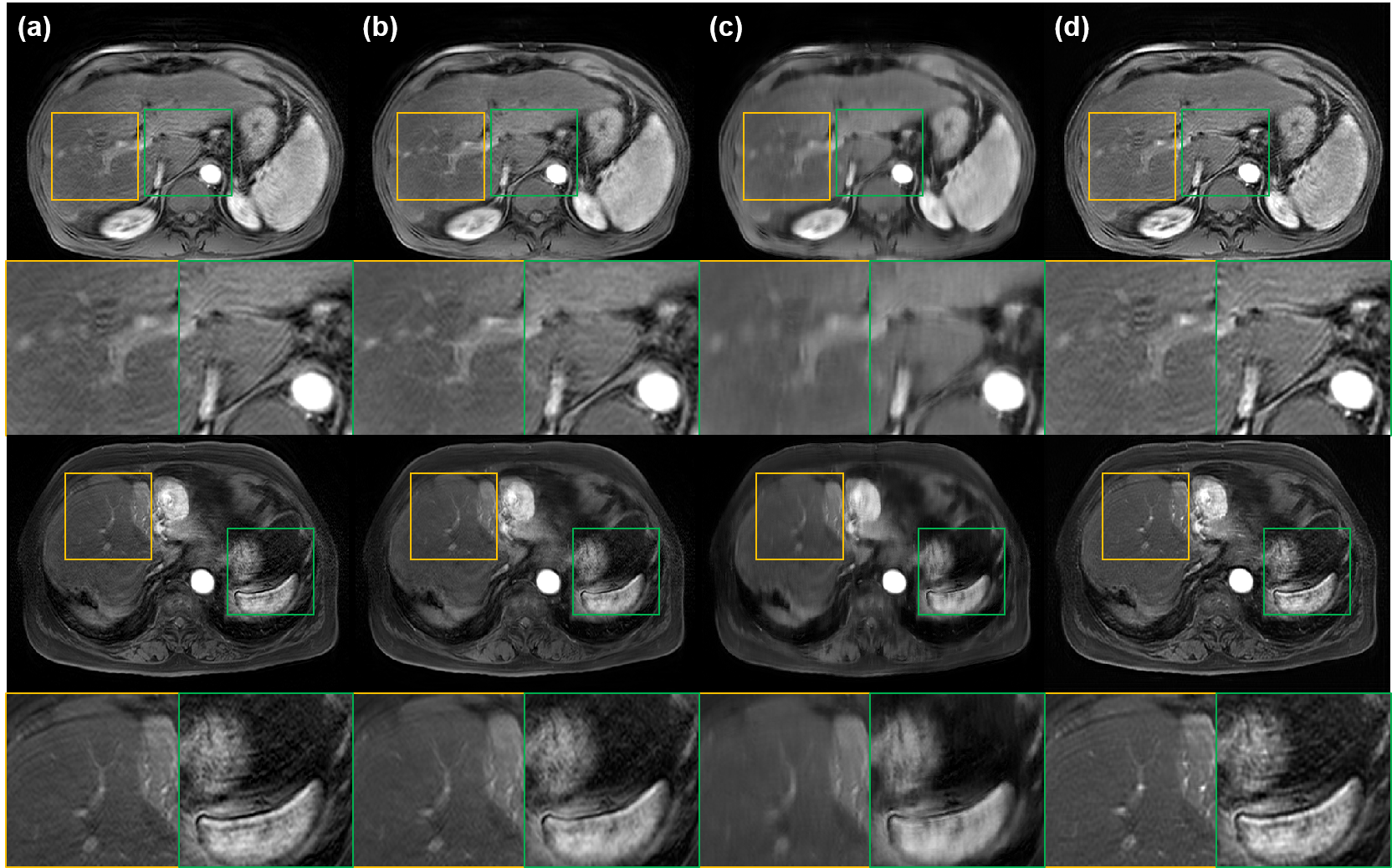}}
	\caption{Motion artifact correction results of liver data with real motion artifact using various methods:
		(a) Artifact images, (b) ours, (c) MARC, and (d) Cycle-MedGAN V2.0.
		The window level of images is adjusted for better visualization.}
	\vspace{-0.5cm}
	\label{fig:real_liver}
\end{figure*}

Table \ref{tbl:quantitative metric} shows the average quantitative metric values of comparison and our methods.
In the simulation experiments with brain data, MARC shows the highest metric values because it is trained through supervised learning with paired data.
However, the quantitative metric values of our method are similar to those of MARC.
Furthermore, our method also shows comparable quantitative results with MARC in the experiments using liver data.

To sum up, we confirm that our method shows similar level of quantitative metric values with MARC, and recovers fine details better than other methods when using simulated motion artifact data.

\begin{table}[!hbt]
	\centering
	\caption{Quantitative comparison of various methods for simulated data. 
		The values in the table are average values for the whole test data.}
	\label{tbl:quantitative metric}
	\resizebox{0.4\textwidth}{!}{
		\begin{tabular}{c | c | c | c} 
			\toprule
			\multicolumn{2}{c|}{ } 																& PSNR (dB)			& SSIM					\\ \midrule\midrule
			\multirow{4}*{Brain}		& Input													& 30.0320			& 0.6215				\\
			\ 							& Proposed												& 34.1347			& 0.9078				\\
			\ 							& MARC \cite{tamada2020motion}							& 34.5972			& 0.9149				\\
			\ 							& Cycle-MedGAN V2.0 \cite{armanious2020unsupervised}	& 26.0323			& 0.8050				\\
			 \midrule
			\multirow{4}*{Liver}		& Input													& 31.7857			& 0.8188				\\
			\							& Proposed												& 31.8611			& 0.8479				\\
			\							& MARC \cite{tamada2020motion}							& 32.2289			& 0.8740				\\
			\							& Cycle-MedGAN V2.0 \cite{armanious2020unsupervised}	& 29.5639			& 0.8033				\\
			 \bottomrule
		\end{tabular}
	}
\end{table}


\subsection{Experiments with Real Data}
Next, we demonstrate our method with real motion artifact data.
The first column of Fig. \ref{fig:real_liver} shows liver MR images with real motion artifact.
When MARC \cite{tamada2020motion} is used to correct real motion artifacts, blurred images are generated (Fig. \ref{fig:real_liver}(c)).
Cycle-MedGAN V2.0 \cite{armanious2020unsupervised} could not catch the difference between two domains of clean images and motion corrupted images.
Accordingly, as shown in Fig. \ref{fig:real_liver}(d), the motion artifact still remains in the output of Cycle-MedGAN V2.0.
On the other hand, our method corrects motion artifact and preserves high frequency details of images as shown in Fig. \ref{fig:real_liver}(b).

\subsection{Clinical Evaluation}
To assess the performance of our model in reducing the arterial phase artifacts of liver MRI acquired during unsuccessful breath-holding, image analyses were performed by a radiologist with 11 years of experience in abdominal MR imaging.
The data set consisted of original and artifact-reduced images using MARC, Cycle-MedGAN V2.0, and our  model, respectively.
All data were subjected to qualitative assessments of the image quality.
To compare the performance of MARC, Cycle-MedGAN V2.0, and our proposed algorithm to reduce artifacts in the hepatic arterial phase, we rated the image quality using a 5-point visual scoring system: 1 = excellent image quality without artifacts; 2 = mild artifacts with satisfactory diagnostic confidence; 3 = moderate artifacts with limited diagnostic confidence; 4 = poor image quality and severe artifacts; 5 = non-diagnostic and marked artifacts with impaired image quality.
The comparison of lesion conspicuity was evaluated between the original and artifact-reduced images and rated using a 4-point conspicuity score: 1 = much better in artifact-reduced images compared to original images, 2= better than in artifact-reduced, 3 = same in both images, and 4 = better in the original than in artifact-reduced images.
We also evaluated the image blurring according the following 4-point scoring system: 1 = no blurring; 2 = mild blurring; 3 = moderate blurring, 4 = severe blurring.

\begin{table}[!hbt]
	\centering
	\caption{Quantitative clinical evaluation results of various methods for in vivo liver data with real motion artifact.}
	\label{tbl:clinical_evaluation}
	\resizebox{0.5\textwidth}{!}{
		\begin{tabular}{c | c | c | c} 
			\toprule
			\                                                       & Artifact              & Blurring              & Lesion conspicuity    \\ \midrule\midrule
			\ Input (original)                                      & 3.20 $\pm$ 1.28       & 1.15 $\pm$ 0.36       & 2.05 $\pm$ 1.05       \\
			\ MARC \cite{tamada2020motion}                          & 1.75 $\pm$ 0.96       & 2.95 $\pm$ 0.22       & 2.40 $\pm$ 1.14       \\
			\ Cycle-MedGAN V2.0 \cite{armanious2020unsupervised}    & 2.35 $\pm$ 1.08       & 1.15 $\pm$ 0.36       & 2.05 $\pm$ 0.99       \\
			\ Proposed                                              & 1.95 $\pm$ 0.94       & 1.20 $\pm$ 0.41       & 1.85 $\pm$ 0.98       \\
			\bottomrule
		\end{tabular}
	}
\end{table}

The results are summarized in Table \ref{tbl:clinical_evaluation}.
The image artifact was significantly reduced by our algorithm, which can be observed from an average motion artifact score $\pm$ standard deviation of 1.95 $\pm$ 0.94 compared to the original MR images of 3.20 $\pm$ 1.28 (p $<$ 0.05).
Compared to the original images, the image blurring significantly increased by MARC seen from an average score $\pm$ standard deviation of 2.95 $\pm$ 0.22 compared to the original MR images of 1.15 $\pm$ 0.36 (p $<$ 0.05).
Our model does not show any increased image blurring.
The performance of the lesion conspicuity shows no significant differences between all models.
However, the hepatic lesion was visualized more clearly in our proposed model (Fig. \ref{fig:clinical_evalutation}).

\begin{figure}[!h]
	\centerline{\includegraphics[width=0.7\linewidth]{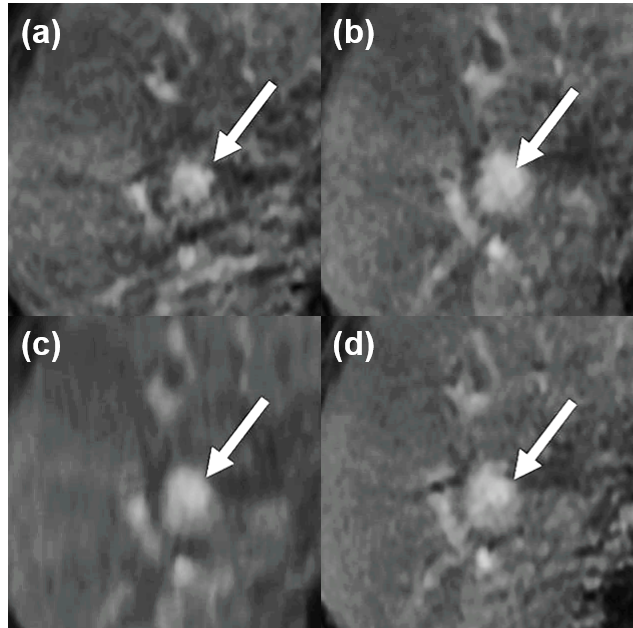}}
	\caption{Lesion conspicuity (a) in the original image, and after application of (b) our model, (c) MARC,  and (d) Cycle-MedGAN V2.0. Axial T1-weighted MR images in patients diagnosed with HCC show improved lesion conspicuity, reduced artifact without significant imaging blurring following our model application compared to the other images.}
	\vspace{-0.5cm}
	\label{fig:clinical_evalutation}
\end{figure}

\section{Discussion}\label{sec:discussion}
\subsection{Effects of Subsampling Direction}
To verify our claim that the motion causes the phase error along the phase encoding direction and our subsampling reduces its effects, we apply our method for experiments with frequency encoding directional subsampling masks.
In these experiments, our network is trained and tested with frequency encoding directional subsampling masks.
Fig. \ref{fig:liver_freq} shows the experimental results with in vivo liver data.
As shown in Fig. \ref{fig:liver_freq}, our method fails to remove motion artifacts when the samples in $k$-space are subsampled along the frequency encoding direction.
Only when the phase encoding directional subsampling is employed, it was possible to correct motion artifacts, which confirms our claim.
Through these experiments, we verify that it is possible to reduce motion artifact using our method because the subsampling along the phase encoding direction can remove some samples with the phase error.

\begin{figure}[!hbt]
	\centerline{\includegraphics[width=\linewidth]{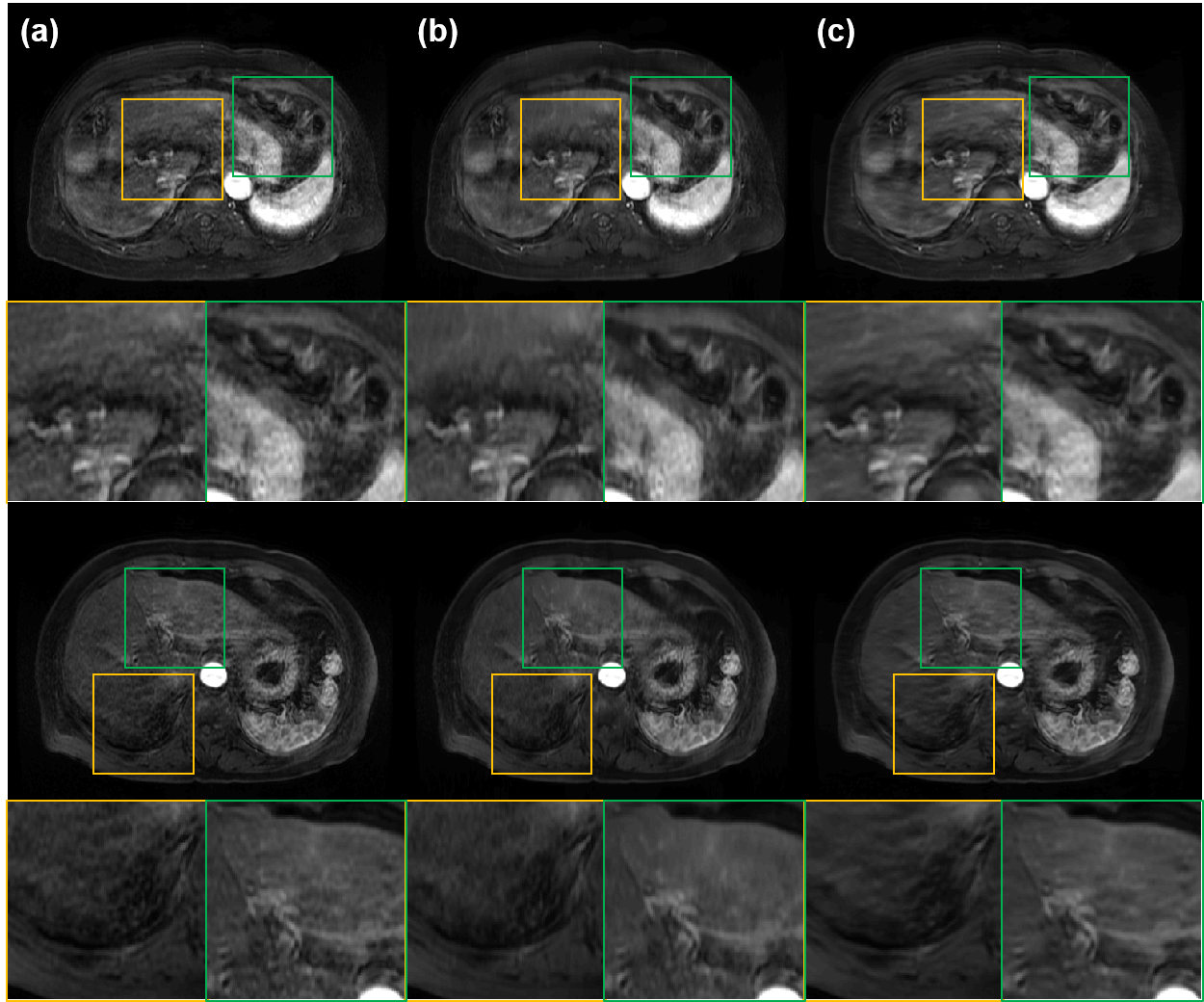}}
	\caption{Motion artifact correction results depending on subsampling directions for in vivo liver MRI data: 
		(a) Artifact images, and correction results by our method with (b) phase encoding directional, and (c) frequency encoding direction subsampling.
		The window level of images is adjusted for better visualization.}
	\vspace{-0.5cm}
	\label{fig:liver_freq}
\end{figure}

\subsection{Effects of Unsupervised Learning}
Next we also verify that an unsupervised training using optimal transport driven cycleGAN is better than supervised learning in terms of blurring and texture refinement.
As shown in Fig. \ref{fig:liver_cycle}, the texture of the organs is better preserved when the unsupervised training  was used than the supervised
learning was used.
In particular, the motion corrected image become blurry when the supervised learning was used.
We believe that this is due to the potential bias in supervised training using the paired data set.
Therefore, we chose the unsupervised training using optimal transport driven cycleGAN although the paired aliased and clean image data are available.

\subsection{Comparison with Other Methods}
In our experiments, we confirmed that MARC \cite{tamada2020motion} shows high quantitative metric values, but generates blurry output.
Because MARC is a supervised method for motion artifact reduction and minimizes $L_1$ loss between clean images and motion artifact images, it shows the highest quantitative metric values.
However, experimental results and clinical evaluation showed that the output of MARC is extremely blurred and it is difficult to distinguish anatomical structures in images which are reconstructed by MARC.
On the other hand, Cycle-MedGAN V2.0 \cite{armanious2020unsupervised} sharpens the edges of images.
Nevertheless, the quantitative results of Cycle-MedGAN V2.0 is lower than the quantitative metrics of images with motion artifact because Cycle-MedGAN V2.0 also highlights motion artifacts.
This maybe because  Cycle-MedGAN V2.0 could not recognize the motion artifact as a difference between two domains.
Furthermore, because Cycle-MedGAN V2.0 is composed of two generators and two discriminators, it requires large amount of memory, long training time, and sensitive hyper-parameter setting.
Actually, we had to reduce the number of channels and layers of networks in Cycle-MedGAN V2.0 to make it even converge for experiments using liver MRI data.

Compared to other methods, our method achieved both of motion artifact correction and restoration of high frequency details.
Although our network is not directly trained using paired data set, it shows competitive quantitative metric compared to MARC.
Moreover, our method outperforms other methods in terms of qualitative results and clinical evaluation, and it also successfully removes real motion artifacts without sacrificing the image quality.

\begin{figure}[!hbt]
	\centerline{\includegraphics[width=0.75\linewidth]{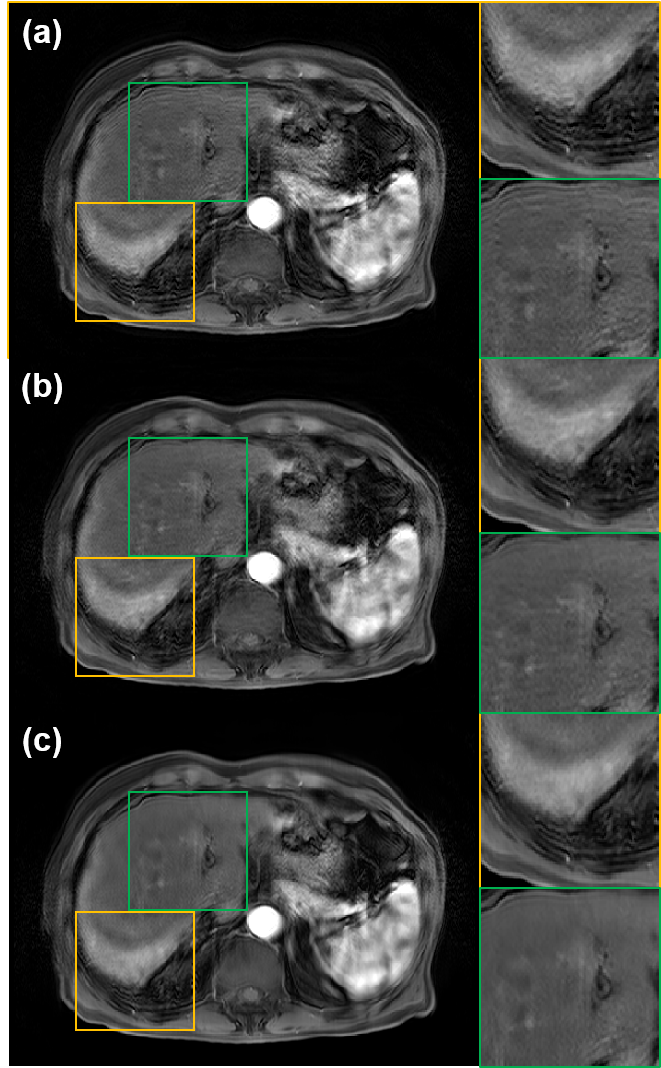}}
	\caption{Ablation study results. 
		(a) Artifact images, and correction results by our method with (b) unsupervised, and (c) supervised trainning.
		The window level of images is adjusted for better visualization.}
	\vspace{-0.5cm}
	\label{fig:liver_cycle}
\end{figure}

\section{Conclusion}\label{sec:conclusion}
In this paper, we proposed a novel MRI motion artifact correction algorithm using the subsampling of $k$-space data.
By converting motion artifact correction problem to a $k$-space outlier-rejecting bootstrap subsampling and aggregation approach for MR reconstruction, it was possible to remove simulated and real motion artefact in MR images.
Moreover, we demonstrated that our method outperforms other existing methods in terms of qualitative and clinical evaluation results.
We believe that our method may be an important platform for MRI motion artifact correction when paired clean data do not exist.
Furthermore, this framework can be easily extended to other artifacts from MRI reconstruction or other medical modalities.

\section{Acknowledgement}
This work was supported by the National Research Foundation (NRF) of Korea grant NRF-2020R1A2B5B03001980.

\bibliographystyle{IEEEtran}
\bibliography{ref,biblio_book}

\end{document}

%% file: macros.tex
\usepackage{bm} 
\usepackage{xspace}

\newcommand{\xmath}[1] {\ensuremath{#1}\xspace}
\newcommand{\blmath}[1] {\xmath{\bm{#1}}}

\newcommand{\xb}{{\blmath x}}

\newcommand{\Lc}{\mathcal{L}}

\newcommand{\Kd}{\mathbb{K}}

\newcommand{\beq}{\begin{equation}}
\newcommand{\eeq}{\end{equation}}
\newcommand{\beqa}{\begin{eqnarray}}
\newcommand{\eeqa}{\end{eqnarray}}